\documentclass{llncs}
\usepackage{makeidx}  
\usepackage{subfig}
\usepackage{graphicx}
\usepackage{amsfonts}
\usepackage{mathtools}
\usepackage{color}

\def\R{\bbbr}

\def\Nor{\mathcal{N}}
\def\I{\mathrm{I}}

\def\m{\mathrm{m}}

\def\F{\mathcal{F}}
\def\G{\mathcal{G}}

\def\1{\mathds{1}}

\def\det{\mathrm{det}}
\def\tr{\mathrm{tr}}

\newcommand\ce[2]{H^{\times}\!\big(#1 \|   #2 \big)}

\newtheorem{observation}{Observation}[section]

\begin{document}
\frontmatter         
\pagestyle{headings}
\title{Detection of elliptical shapes via cross-entropy clustering}
\titlerunning{Cross-entropy clustering} 

\author{Jacek Tabor\inst{1} \and Krzysztof Misztal\inst{1,2}}

\authorrunning{Jacek Tabor et al.} 
\institute{
Jagiellonian University\\
Faculty of Mathematics and Computer Science\\
\L{}ojasiewicza 6, 30-348 Krak\'ow, Poland\\
{\sf tabor@ii.uj.edu.pl}\and
AGH University of Science and Technology\\
Faculty of Physics and Applied Computer Science\\
al. A. Mickiewicza 30, 30-059 Krak\'ow, Poland\\
{\sf Krzysztof.Misztal@fis.agh.edu.pl}
}

\maketitle             

\begin{abstract}
The problem of finding elliptical shapes in an image will be considered.
We discuss the solution which uses cross-entropy clustering.
The proposed method allows the search for ellipses with predefined 
sizes and position in the space. Moreover, it works
well for search of ellipsoids in higher dimensions.

\end{abstract}
\section{Introduction}
Ellipse detection is one of the most important problems in image processing.
It has been researched using a good variety of methods, see i.e. Tsuji and Matsumoto \cite{tsuji1978}, Davies \cite{Davies1989}. Most of the existing techniques use the Hough Transform \cite{Illingworth1988} -- that is very
memory and time consuming. 

In this paper a new approach will be presented
and its advantages and disadvantages will be discussed. 
We show the results of the algorithm on the  
pictures from Fig. \ref{fig:infull}.
\begin{figure}
  \centering
  \subfloat[]{\label{fig:t1full}\fbox{\includegraphics[width=0.3\textwidth]{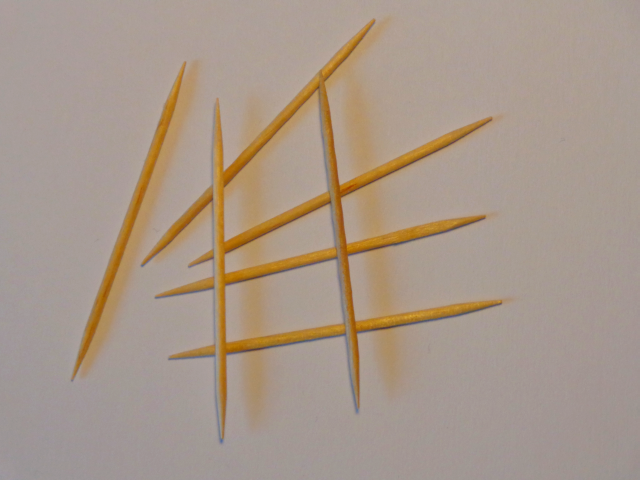}}}~
  \subfloat[]{\label{fig:t1bin}\fbox{\includegraphics[width=0.3\textwidth]{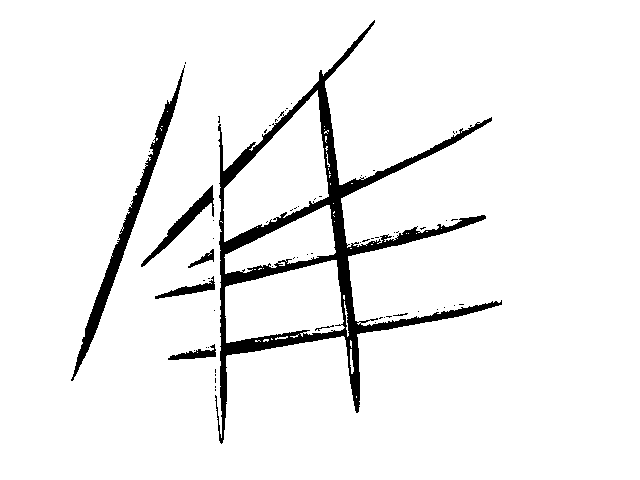}}}~
  \subfloat[]{\label{fig:t1clus}\fbox{\includegraphics[width=0.3\textwidth]{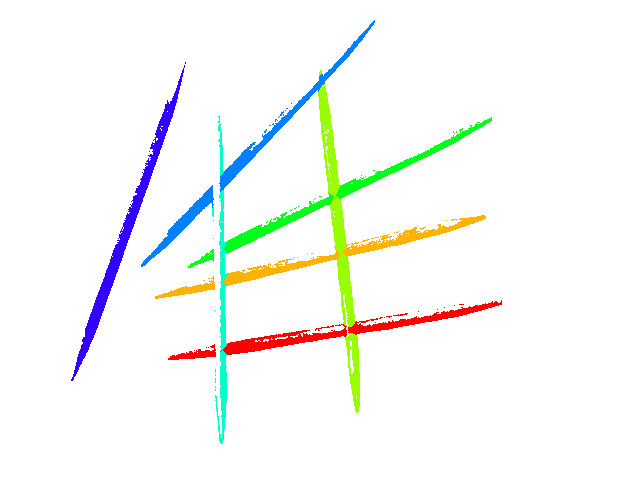}}}
  \caption{The result of our algorithm: Fig \ref{fig:t1full} -- original image, Fig \ref{fig:t1bin} -- binarized image, the input for our algorithm, Fig \ref{fig:t1full} -- outcome form algorithm, clusters marked in different colors.
}
  \label{fig:infull}
\end{figure}
The algorithm discussed in this paper:
\begin{itemize}
\item is easily adaptable, ie. if we know the expected shape 
of the object sought, or its position (orientation) in space, by little calculation we can prepare a proper configuration for its detection;
\item can detect simultaneously multiple type of objects, ex. we can look for matches 
and coins 
at the same time;
\item is rather insensitive to the disturbance of the picture (such as bluring, contrast and illumination modification, etc);
\item can be used for classification (we can detect specified shapes) and for clustering (we can use it for exploring the data structure).
\end{itemize}
The acceptable disadvantage of the presented method
is that to work well we need the beforehand knowledge
that on the picture  we study there are no other objects
than ellipse-like shapes. Consequently, our approach 
is well-adapted for example to the following tasks:
\begin{itemize}
\item count the number of ellipses on the picture;
\item divide the shapes into circles of different radiuses;
\item count the number of vertical and horizontal ellipses. 
\end{itemize}

Our idea uses a cross-entropy clustering \cite{tabor2013} (CEC), which from the practical point of view can be seen as joining of the k-means method with the model approach used in expectation maximization (EM).
EM \cite{Mc-Kr,celeux1995gaussian,mcnicholas2008parsimonious}
is one of the basic and most important applications of maximal likelihood in the density estimations \cite{Le-Ca}. 
EM, or its variations like classification EM \cite{same2007online} are often applied in clustering. Although EM approach is quite general, and gives good results, to apply it we usually need to first perform complicated 
computations. Moreover,  to accomplish the M step one commonly needs numerically consuming minimization techniques, and consequently
EM is relatively slow and cannot deal well with large data.

Our aim in this paper is to show that CEC is well-adapted to classification and detection of ellipses and ellipsoids.
The advantage of CEC over EM is simplicity and speed -- in the
case of typical Gaussian families we do not
need the M-step, which enables us in particular to
use fast and efficient Hartigans approach. Moreover, as the use of every cluster
in CEC has its cost, contrary to classification EM,
CEC reduces on-line clusters which carry no
information, which in practice implies that
our algorithm can find the ``right'' number of ellipses
on the picture.

Let us discuss the contents of the paper.
In the first part of our work we briefly describe the CEC algorithm. 
In the next section we present the basic models we use (compare with \cite{fraley1998algorithms}).
We also present results of numerical
experiments. 
Then we describe the procedure for finding toothpicks in the image (see  Fig. \ref{fig:infull}). 

In Appendix we provide the proof
of the only cross-entropy formula from section which 
is essentially new. In our opinion its proof is worth
including as in fact it given a method which can be easily
used in search for cross-entropy in other Gaussian subfamilies.

\section{Theoretical background of CEC}

In this section we give a short introduction to CEC, for 
more detailed explanation we refer the reader to \cite{tabor2013}.
To explain CEC we need to introduce the "energy function"
we want to minimize. By the cross-entropy of the probability measure $\mu$ (which represent the data-set we study) with respect to density 
$f$ we understand
$$
\ce{\mu}{f}=-\int_{\R^N} \ln f(y) \, d\mu(y).
$$
The above cross-entropy corresponds to the theoretical code-length 
of compression of $\mu$-randomly chosen element of $\R^N$ with
the code optimized for density $f$ \cite{Co-Th}.
In a more general case when one is interested in (best) coding for 
$\mu$ by densities chosen from family $\F$, we arrive at {\em the cross-entropy of $\mu$ with respect to a family of coding densities $\F$}
$$
\ce{\mu}{\F}:=\inf_{f \in \F}\ce{\mu}{f}.
$$
In the case of splitting of $\R^N$ into pairwise disjoint sets $U_1,\ldots,U_N$
such that elements of $U_i$ we "code" by optimal density from 
family $\F_i$, the mean code-length of randomly chosen element $x$ equals
\begin{equation} \label{e0}
E_{\mu}(U_1,\F_1;\ldots;U_n,\F_n):=\sum_{i=1}^k
\mu(U_i) \cdot (-\ln(\mu(U_i))+\ce{\mu_{U_i}}{\F_i}),
\end{equation}
where $\mu_{U}$ denotes the normalized restriction of $\mu$ to the
set $U$ and is given by $\mu_U(A):=\frac{1}{\mu(U)}\mu(A \cap U)$.

The aim of CEC is to find
splitting of $\R^N$ into pairwise disjoint sets $U_i$ which minimize
the function given in \eqref{e0}. In this paper we restrict for the sake of 
simplicity to clusters generated by Gaussian densities (although one can easily 
use any density family for which MLE can be performed). 


Now we proceed with discussion of the Gaussian models
we will use in CEC. We consider following density families:
\begin{enumerate}
\item $\G_{\Sigma}$ -- Gaussian densities with covariance $\Sigma$. The 
clustering will have the tendency to divide the data into clusters resembling
the unit circles in the Mahalanobis distance given by $\|x-y\|_{\Sigma}^2:=
(x-y)^T\Sigma(x-y)$. Its particular important subfamily is given by $\G_{r\I}$,
where $r>0$ is fixed (in this case we will have tendency to divide the data
into "circles" with approximate radius of $\sqrt{r}$).
\item $\G_{(\cdot I)}$ -- spherical Gaussian densities, which covariance is
proportional to identity. The clustering will try to divide the data into
circles of arbitrary sizes.
\item $\G_{\mathrm{diag}}$ -- Gaussians with diagonal covariance. The 
clustering will try to divide the data into ellipsoid with radiuses 
parallel to coordinate axes.
\item $\G$ -- all Gaussian densities. In this case we divide dataset into ellipsoid-like clusters without any preferences
concerning the size or shape or position in space of the ellipsoid.
\end{enumerate}

We need a result which says what is the cross-entropy of the probability measure $\mu$ with respect to coding adapted for the 
respective Gaussian subfamilies. A basic role is played by the following
observation.

\begin{observation} \label{ob}
Let $\mu$ be a discrete or continuous probability measure 
in $\R^N$ with well-defined
mean $m_{\mu}:=\int xd\mu(x)$ and covariance matrix $\Sigma_{\mu}:=
\int (x-m_{\mu}) (x-m_{\mu})^Td\mu(x)$.
Let a fixed positive-definite symmetric matrix $\Sigma$ be given.

Then $\ce{\mu}{\G_{\Sigma}}=H^{\times}\big(\mu_{\G}\|\Nor{(\m_{\mu},\Sigma)}\big)$,
where $\mu_\G$ denotes the probability measure with Gaussian density
of the same mean and covariance as $\mu$. Consequently
\begin{equation} \label{e1}
\ce{\mu}{\G_{\Sigma}}=\frac{N}{2} \ln(2\pi)+\frac{1}{2}\tr(\Sigma^{-1}\Sigma_{\mu})+\frac{1}{2}\ln \det(\Sigma).
\end{equation}
\end{observation}

By applying the above proposition one can easily 
deduce\footnote{In practice all the formulas given in the 
are known, see for example \cite{tabor2013}.} 
the formulas for cross-entropy given the Table \ref{tab1:cec}.

\vspace{-2ex}

\begin{table}\centering
\begin{tabular}{||l|l|l||} \hline \hline
$\F$ &  cov. matrix & $\ce{\mu}{\F}$   \\[0.5ex] 
\hline \hline

$\G_{\Sigma}$ & $\Sigma$ & $\frac{N}{2} \ln(2\pi)+\frac{1}{2}\tr(\Sigma^{-1}\Sigma_{\mu})+\frac{1}{2}\ln \det(\Sigma)$
\\[0.5ex] \hline

$\G_{r\I}$ & $r\I$ &
$\frac{N}{2}\ln(2\pi)+\frac{1}{2r}\tr(\Sigma_{\mu})+\frac{N}{2}\ln r$ \\[0.5ex] \hline

$\G_{(\cdot\I)}$ & $\frac{\tr(\Sigma_{\mu})}{N} \I$ & $\frac{N}{2}\ln(2\pi e/N)+\frac{N}{2}\ln (\tr \Sigma_\mu)$ \\[0.5ex] \hline

$\G_{\mathrm{diag}}$ & $\mathrm{diag}(\Sigma)$ & $\frac{N}{2}\ln(2\pi e)+\frac{1}{2}\ln(\det(\mathrm{diag}(\Sigma_\mu)))$ \\[0.5ex]
\hline


$\G$ & $\Sigma_{\mu}$ & $\frac{N}{2}\ln(2\pi e)+\frac{1}{2}\ln \det(\Sigma_{\mu})$ \\[0.5ex] \hline \hline
\end{tabular}
\caption{Table of cross-entropy formulas with respect to Gaussian subfamilies.}
\label{tab1:cec}
\end{table}

\vspace{-3ex}

In the second column we give the formula for the 
covariance matrix of the Gaussian density which realizes the desired minimum of cross-entropy (obviously the mean is always the mean of the measure).
Simple applications of the formulas given above can be found on the Figure
\ref{fig:ir}.

\setlength\fboxsep{0pt}
\setlength\fboxrule{1pt}
\begin{figure}
  \centering
  \subfloat[]{\label{fig:p1in}\fbox{\includegraphics[width=0.25\textwidth]{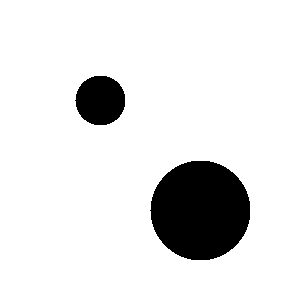}}}
  \subfloat[]{\label{fig:p1out}\fbox{\includegraphics[width=0.25\textwidth]{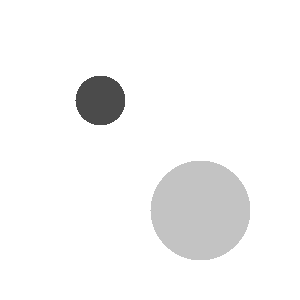}}}
  \subfloat[]{\label{fig:p3in}\fbox{\includegraphics[width=0.25\textwidth]{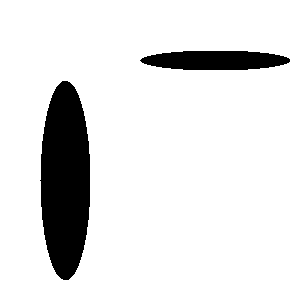}}}
  \subfloat[]{\label{fig:p3out}\fbox{\includegraphics[width=0.25\textwidth]{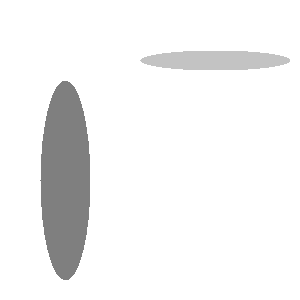}}}
  \caption{The simplest case: input and outcome for our algorithm applied to $\G_{r\I}$ (Fig. \ref{fig:p1in} and \ref{fig:p1out})  and $\G_{\mathrm{diag}}$ (Fig. \ref{fig:p3in} and \ref{fig:p3out}).}
  \label{fig:ir}
\end{figure}


\section{Case study}

Let us explain the method on the following simple problem:
assume that we want to count the toothpicks
on the Fig. \ref{fig:irisx}. To do so we take a particular object
and compute its covariance matrix. We have obtained 
a covariance with eigenvalues
$$
\lambda_1=4938.5	 \mbox{ and } \lambda_2=5.7.
$$
Since we want to allow the toothpick to have any position in space,
we introduce the set $\G_{\lambda_1,\lambda_2}$ to consist
of all Gaussian densities on the plane with covariance
matrix having eigenvalues $\lambda_1$
and $\lambda_2$ (observe that this set is rotation and translation 
invariant, but not scale invariant). 

Consider now a probability measure $\mu$, representing our data,
with covariance $\Sigma_{\mu}$, with eigenvalues $\lambda_1^{\mu}>\lambda_2^{\mu}>0$.
By applying Proposition \ref{bas} (see Appendix) jointly 
with Observation \ref{ob} we easily conclude
that the best approximation
(understood in the maximal likelihood or equivalently cross-entropy, sense)
of $\mu$ in $\G_{\lambda_1,\lambda_2}$ is given by
the Gaussian density with covariance matrix with the same
eigenvectors as $\Sigma_{\mu}$ and eigenvalues $\lambda_1$
and $\lambda_2$.  
\begin{figure}
  \centering
  \subfloat[]{\label{fig:15full}\fbox{\includegraphics[width=0.35\textwidth]{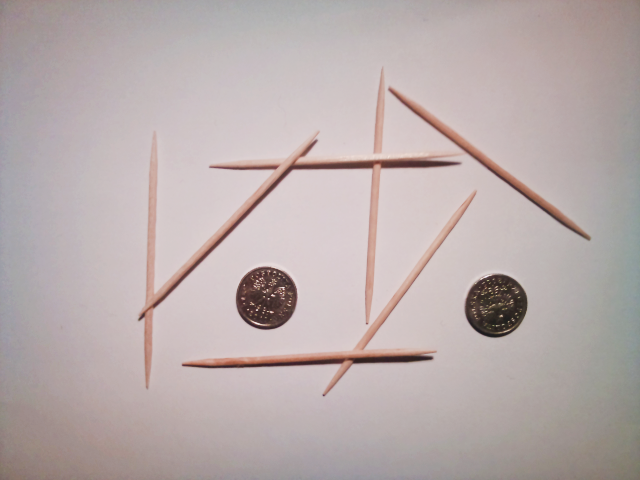}}}~
  \subfloat[]{\label{fig:15bin}\fbox{\includegraphics[width=0.35\textwidth]{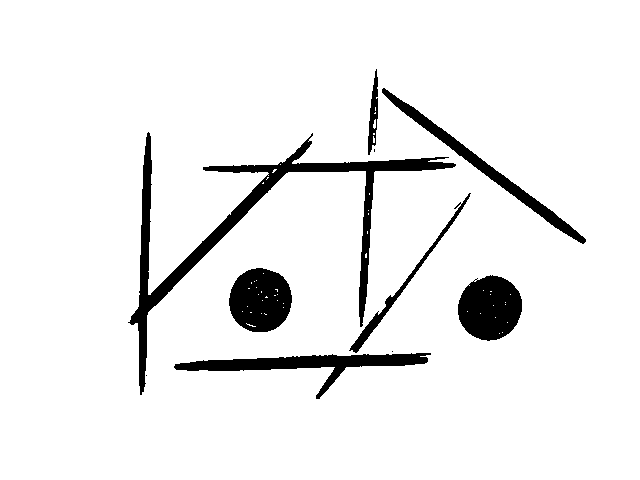}}}\\
  \subfloat[]{\label{fig:15clus}\fbox{\includegraphics[width=0.35\textwidth]{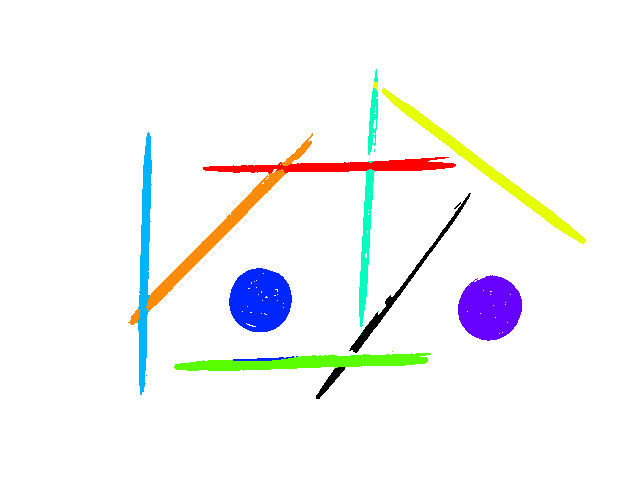}}}~
  \subfloat[]{\label{fig:15el}\fbox{\includegraphics[width=0.35\textwidth]{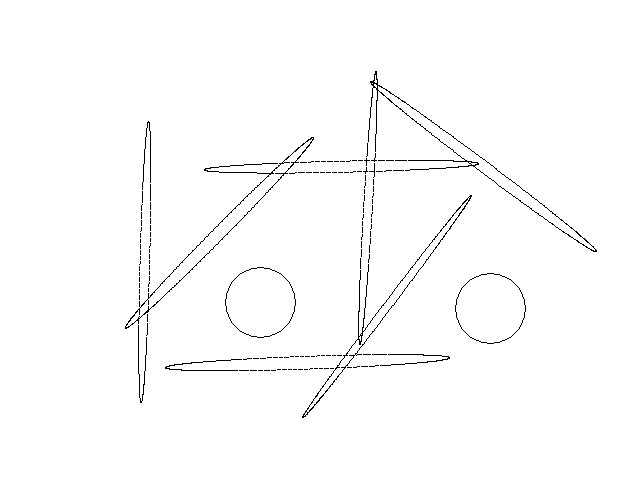}}}
  \caption{The result of our algorithm: Fig \ref{fig:15full} -- original image, Fig \ref{fig:15bin} -- binarized image, the input for our algorithm, Fig \ref{fig:15clus} -- outcome form algorithm, clusters marked in different colors, Fig \ref{fig:15el} -- outcome form algorithm, ellipses with the same mean and covariance as calculated by algorithm densities.}
  \label{fig:irisx}
\end{figure}
Consequently, the cross-entropy, which plays the role of energy,
$\ce{\mu}{\G_{\lambda_1,\lambda_2}}$ thanks to \eqref{e1} is
given by
$$
\ce{\mu}{\G_{\lambda_1,\lambda_2}}=\frac{N}{2} \ln(2\pi)+\frac{1}{2}
(\lambda_1^{\mu}/\lambda_1+\lambda_2^{\mu}/\lambda_2)+\frac{1}{2}(\ln(\lambda_1)+\ln(\lambda_2)).
$$
By applying Hartigan approach we can now find the splitting of the data into pairwise disjoint sets $U_1,\ldots,U_k$ which minimizes
the value of \eqref{e0}. Results of our method can be seen on Figure
\ref{fig:irisx} (we omit here the natural preliminary binarization procedure).

To visualize the found clusters, we 
draw the boundary of an ellipse with the same mean and covariance as a given 
density estimator\footnote{We recall that covariance matrix of a uniform 
density of an ellipse with radiuses $r_1,r_2$ is given by $[r_1^2/4,0;0,r_2^2/4]$, that is we draw the ellipse with radiuses $2\sqrt{\lambda_i}$.}.

\section{Conclusion}

We have proposed a new method, which uses
cross-entropy clustering approach, to classification and detection
of ellipse-like shapes. The main advantage of the 
method lies in the fact that it can be easily adapted
to finding ellipses of desired shape and position in space.
The basic disadvantage is that in current algorithm configuration (basic approach) we can deal only with
pictures which contain only ellipse-like shapes (for example
we cannot discover ellipses in a picture with ellipses and rectangles). Our further work will consist on elimination of this inconvenience.

\section{Appendix: how to compute MLE for Gaussian families}

The situation is very simple if we search for the MLE, or in other words
for the minimum in \eqref{e1} in the class of diagonal matrices (subclass consisting of Gaussians with independent variables). A more 
requiring and difficult question is to find the desired
minimum in the class of all Gaussians. Below we present 
an approach which allows to do this.

We will use the well-known von Neumann trace inequality \cite{Gr,Mi}:

\medskip

\noindent{\bf Theorem [von Neumann trace inequality].} {\em Let $E,F$ be complex $N \times N$ matrices. Then
\begin{equation} \label{neu}
|\tr(EF) | \leq \sum_{i=1}^N s_i(E)\cdot s_i(F),
\end{equation}
where $s_i(D)$ denote the ordered (decreasingly)
singular values of matrix $D$.}

\medskip

Let us recall that for the symmetric positive matrix its
eigenvalues coincide with singular values.

Given $\lambda_1,\ldots,\lambda_N \in \R$ by $S_{\lambda_1,\ldots,\lambda_N}$ we denote the set of all symmetric matrices with eigenvalues $\lambda_1,\ldots,\lambda_N$. The following proposition plays
the basic role in the search for optimal Gaussian densities,
as it reduces the search from all symmetric matrices to search
in the set of eigenvalues. Since its proof is short, we provide it for the sake of completeness.

\begin{proposition} \label{bas}
Let $B$ be a symmetric nonnegative matrix with eigenvalues $\beta_1 \geq \ldots \geq\beta_N \geq 0$ .
Let $0 \leq \lambda_1 \leq \ldots \leq \lambda_N$ be fixed.
Then 
$$
\min_{A \in S_{\lambda_1,\ldots,\lambda_N}} \tr(AB)=\sum_i \lambda_i \beta_i.
$$ 
\end{proposition}

\begin{proof}
Let $e_i$ denote the orthogonal basis build from the eigenvectors of $B$, and let 
operator $\bar A$ be defined in this base by $\bar A(e_i)=\lambda_i e_i$. Then trivially 
$$
\min_{A \in S_{\lambda_1,\ldots,\lambda_N}} \tr(AB) \leq
\tr(\bar AB)=\sum_i \lambda_i \beta_i.
$$

To prove the inverse inequality we will use the
von Neumann trace inequality. Let $A \in S_{\lambda_1,\ldots,\lambda_N}$ be arbitrary. We apply the inequality \eqref{neu} for $E=\lambda_N \I-A$, $F=B$.
Since $E$ and $F$ are symmetric nonnegatively defined matrices, their eigenvalues
$\lambda_N-\lambda_i$ and $\beta_i$ coincide with singular values, and therefore by \eqref{neu}
\begin{equation} \label{nu2}
\tr((\lambda_N\I-A)B) \leq \sum_i(\lambda_N-\lambda_i)\beta_i=
\lambda_N \sum_i \beta_i -\sum_i \lambda_i \beta_i.
\end{equation}
Since $\tr((\lambda_N\I-A)B)=\lambda_N \sum_i \beta_i -\tr(AB)$,
from inequality \eqref{nu2} we obtain that
$\tr(AB) \geq \sum_i \lambda_i \beta_i$.
\end{proof}

\begin{corollary}
Assume that we want to find the best fit of $\mu$ 
with covariance $\Sigma_{\mu}$ in the class $\G_{\lambda_1,
\ldots,\lambda_n}$, where $\lambda_1 \geq \ldots \geq \lambda_n>0$. 

To do so we take the eigenvalues $\lambda_1^{\mu} \geq \ldots 
\geq \lambda_n^{\mu}$ corresponding to orthonormal eigenvectors
$e_1^{\mu},\ldots,e_n^{\mu}$, and then $\Sigma$ is given
in the base as a diagonal matrix with $\lambda_1,\ldots,\lambda_n$ on 
the diagonal.
\end{corollary}

%
%


\begin{thebibliography}{10}

\bibitem{celeux1995gaussian}
Celeux, G. and Govaert, G.:
\newblock Gaussian parsimonious clustering models.
\newblock Pattern Recognition, 28(5):781--793, 1995.

\bibitem{Co-Th}
Cover, T.M. and Thomas, J.A. and Wiley, J. et~al.:
\newblock Elements of information theory, volume~6.
\newblock Wiley Online Library, 1991.

\bibitem{Davies1989}
Davies, E.R.:
\newblock Finding ellipses using the generalised Hough transform.
\newblock Pattern Recognition Letters, 9 (2), 87--96, 1989.

\bibitem{fraley1998algorithms}
Fraley, C.:
\newblock Algorithms for model-based Gaussian hierarchical clustering.
\newblock SIAM Journal on Scientific Computing, 20(1):270--281, 1998.

\bibitem{fimg}
Free Stock Images.
\url{http://www.stockfreeimages.com/}

\bibitem{Gr}
Gr{\"u}nwald, P.D. and Myung, I.J. and Pitt, M.A.:
\newblock Advances in minimum description length: Theory and applications.
\newblock the MIT Press, 2005.

\bibitem{Illingworth1988}
Illingworth, J. and Kittler, J.:
\newblock A survey of the Hough transform
Computer vision, graphics, and image processing, 44 (1), 87--116, 1988.

\bibitem{Le-Ca}
Lehmann, E.L. and Casella, G.:
\newblock Theory of point estimation, volume~31.
\newblock Springer Verlag, 1998.

\bibitem{Mc-Kr}
McLachlan, G.J. and Krishnan, T.:
\newblock The EM algorithm and extensions.
\newblock 274, Wiley New York, 1997.

\bibitem{mcnicholas2008parsimonious}
Mcnicholas, P.D. and Murphy, T.B.:
\newblock Parsimonious gaussian mixture models.
\newblock Statistics and Computing, 18(3):285--296, 2008.

\bibitem{Mi}
Mirsky, L.:
\newblock A trace inequality of John von Neumann.
\newblock Monatsh. Math., 79 (4):303--306, 1975.

\bibitem{same2007online}
Sam{\'e}, A. and Ambroise, C. and Govaert, G.:
\newblock An online classification em algorithm based on the mixture model.
\newblock Statistics and Computing, 17(3):209--218, 2007.

\bibitem{tabor2013}
Tabor, J. and Spurek, P.:
\newblock Cross-entropy clustering.
\newblock Available from \url{http://arxiv.org/pdf/1210.5594.pdf}, 2012.

\bibitem{tsuji1978}
Tsuji, S. and Matsumoto,  F.:
\newblock Detection of ellipses by a modified Hough transformation
\newblock IEEE Trans. Comput., C-27 (1978), pp. 777--781


\end{thebibliography}
\end{document}